\documentclass{ecai} 

\usepackage{latexsym}
\usepackage{amssymb}
\usepackage{amsmath}
\usepackage{amsthm}
\usepackage{booktabs}
\usepackage{enumitem}
\usepackage{graphicx}
\usepackage{color}
\usepackage{float}

\newtheorem{theorem}{Theorem}

\newtheorem{proposition}[theorem]{Proposition}

\newtheorem{definition}{Definition}

\newcommand{\BibTeX}{B\kern-.05em{\sc i\kern-.025em b}\kern-.08em\TeX}

\begin{document}


\begin{frontmatter}


\paperid{515} 


\title{PAGE: Parametric Generative Explainer for Graph Neural Network}


\author[A]{\fnms{Yang}~\snm{Qiu}}
\author[A]{\fnms{Wei}~\snm{Liu}}
\author[B]{\fnms{Jun}~\snm{Wang} \thanks{Corresponding Author~(Email: jwang@iwudao.tech).}} 
\author[A]{\fnms{Ruixuan}~\snm{Li}\thanks{Corresponding Author~(Email: rxli@hust.edu.cn). This paper is a collaboration between Intelligent and Distributed Computing Laboratory, Huazhong University of Science and Technology and iWudao Tech.}}

\address[A]{School of Computer Science and Technology, Huazhong University of Science and Technology, China}
\address[B]{iWudao Tech}


\begin{abstract}
This article introduces PAGE, a \underline{pa}rameterized \underline{ge}nerative interpretive framework. PAGE is capable of providing faithful explanations for any graph neural network without necessitating prior knowledge or internal details. Specifically, we train the autoencoder to generate explanatory substructures by designing appropriate training strategy. Due to the dimensionality reduction of features in the latent space of the autoencoder, it becomes easier to extract causal features leading to the model's output, which can be easily employed to generate explanations. To accomplish this, we introduce an additional discriminator to capture the causality between latent causal features and the model's output. By designing appropriate optimization objectives, the well-trained discriminator can be employed to constrain the encoder in generating enhanced causal features. Finally, these features are mapped to substructures of the input graph through the decoder to serve as explanations. Compared to existing methods, PAGE operates at the sample scale rather than nodes or edges, eliminating the need for perturbation or encoding processes as seen in previous methods. Experimental results on both artificially synthesized and real-world datasets demonstrate that our approach not only exhibits the highest faithfulness and accuracy but also significantly outperforms baseline models in terms of efficiency. 
\end{abstract}

\end{frontmatter}


\section{Introduction}



Graph Neural Networks (GNNs) have consistently demonstrated promising results across a wide variety of tasks~\citep{wuComprehensiveSurveyGraph2021,xuHowPowerfulAre2019,ZHOU202057}. As GNNs find applications in crucial domains where trustworthy AI is imperative, the need for their interpretability has become increasingly paramount~\citep{yuanExplainabilityGraphNeural2023}.
GNNExplainer is the first general model-agnostic approach for interpreting GNNs, and searches for soft masks for edges and node features to explain the predictions via mask optimization~\citep{yingGNNExplainerGeneratingExplanations2019a}.
Since GNNExplainer largely focuses on providing the local interpretability by generating a painstakingly customized explanation for a single instance individually and independently, it is not sufficient to provide a global understanding of the trained model~\citep{luoParameterizedExplainerGraph2020a}. 
Furthermore, GNNExplainer has to be retrained for every single explanation. So, in real-world scenarios, GNNExplainer would be time-consuming and impractical.
To address the above issues, PGExplainer was proposed to learn succinct underlying structures as the explanations from the observed graph data~\citep{luoParameterizedExplainerGraph2020a}. It also models the underlying structure as edge distributions, where the explanatory subgraph is sampled using binary hard masks for edges.
PGExplainer utilizes a deep neural network to parameterize the explanation selection process, providing the ability to collectively explain multiple instances, so it has better generalization ability and can be utilized in an inductive setting easily~\citep{luoParameterizedExplainerGraph2020a}.
Since there is no ground truth for explanatory subgraphs, PGExplainer has to employ reinforcement learning or Gumbel-Softmax sampling to search for informative explanatory subgraphs within a vast space that grows exponentially with the number of edges. Consequently, the computational complexity remains relatively high~\citep{linGenerativeCausalExplanations2021}.

Instead of using edge masks, recent methods have attempted to produce the adjacency matrix of explanatory subgraphs directly through a Variational
Graph Auto-encoder (VGAE)~\citep{kipfVariationalGraphAutoEncoders2016}, eliminating the need to consider each node or edge individually and achieving greater efficiency. 
GEM is the pioneering work in this domain. It introduced a distillation process grounded in the concept of Granger causality~\citep{bresslerWienerGrangerCausality2011} to generate ground-truth explanatory subgraphs for the training of VGAE~\citep{linGenerativeCausalExplanations2021}. However, the distillation process intrinsically assumes independence between the edges. This could be problematic as graph-structured data is inherently interdependent~\citep{linOrphicXCausalityInspiredLatent2022}.
Differing from GEM, which quantifies the causal influence from edges, OrphicX~\citep{linOrphicXCausalityInspiredLatent2022} suggests identifying the underlying causal factors from the latent space, allowing it to bypass direct interaction with intricately interdependent edges.
It divides the latent representation produced by the encoder into causal features and spurious features and identifies the underlying causal features by harnessing information flow measurements~\citep{ayINFORMATIONFLOWSCAUSAL2008}, quantifying the causal information emanating from the latent features. It then constructs the adjacency matrix for the explanatory subgraph based on these causal features~\citep{linOrphicXCausalityInspiredLatent2022}.

The fundamental architecture of Gem and OrphicX is depicted in Figure~\ref{figure1}. Gem, while illustrating the viability of training VGAE as a graph interpreter, is subject to certain constraints. Gem's methodology involves a distillation process aimed at generating ground-truth subgraphs for training the autoencoder. This process perturbs the input graph at the edge level and assesses the importance of each edge based on the resulting reduction in model error. However, this approach assumes edge independence, failing to adequately explore the ability to discern causal features within the autoencoder's latent space.
OrphicX endeavors to utilize information flow estimation to evaluate the causal relationships between each feature dimension in the latent space and the model predictions. This strategy aims to bypass perturbations from input dimensions. Nonetheless, due to computational constraints and complexity considerations, OrphicX operates within a severely restricted sampling range for each sample. 
For instance, regardless of the size of a single input graph, OrphicX samples only five nodes of the graph to estimatie the causal effects of the whole graph on the output, which is obviously a limitation that undermines practicality and necessary adaptability. Moreover, the efficacy of information flow measurements in OrphicX appears limited. This assertion is supported by our experimental findings, which indicate that excluding the information flow term from the loss function has minimal impact on the model's performance.
Furthermore, both the distillation process in Gem and the information flow measurement in OrphicX necessitate numerous perturbations or samplings on individual samples. These procedures incur significant time and space complexity, posing challenges to the application of these methods on large-scale graphs and extensive datasets.

To address these problems in OrphicX, we introduce a novel and more efficient model termed the \underline{Pa}rametric \underline{Ge}nerative Explainer (PAGE\footnote{Code is available at \textit{https://github.com/anders1123/PAGE/}}).
Specifically, our optimization objective remains to maximize the mutual information between explanations and original predictions, but we introduce an extra parameterized discriminator to acquire the global understanding of the causal features in the latent space. 
As illustrated in Figure~\ref{figure1}(c) and Figure~\ref{figure2}, our model consists of an autoencoder and an additional discriminator. The features in the latent space compressed by the encoder are divided into two parts: causally relevant features related to predictions are used to generate explanations, while non-causal features are discarded. The training process consists of two stages: In the first stage, autoencoder and discriminator are trained together. The well-trained parameterized discriminator is designed to maximize the mutual information between causal features and prediction results. Parameterized discriminator eliminates the need for information flow estimation through sampling. Furthermore, since the discriminator is well-trained, it possesses a global comprehension on the entire dataset. In the second stage, the parameters of the discriminator are frozen to constrain the encoder to learn better causal features. Gem requires calculating causal contributions for each edge of a instance, while OrphicX involves extensive sampling across different dimensions including samples and features. In comparison, our approach is significantly more efficient than these methods. The trained explainer can collectively generate explanations, eliminating the need for extensive sampling and computations. 

Our main contributions can be outlined as follows:

1) We introduce PAGE, a novel generative GNN explanation method. This method replaces complicated sampling processes with a streamlined discriminator, enhancing effectiveness.

2) We further refine the optimization objective, eliminating the need of complex perturbation or sampling process. Instead, we directly align predictions from the original graph with those from the explanatory subgraph.

3) Our research included experiments spanning both node and edge levels, conducted on both synthesized and real-world datasets. Experimental results clearly show that our approach outperforms existing methods across various metrics. This includes, but is not limited to, confidence and accuracy. Moreover, our method boasts significantly higher efficiency when juxtaposed with prior methods.

\begin{figure}[t]
    \centering
    \includegraphics[width=1.0\columnwidth]{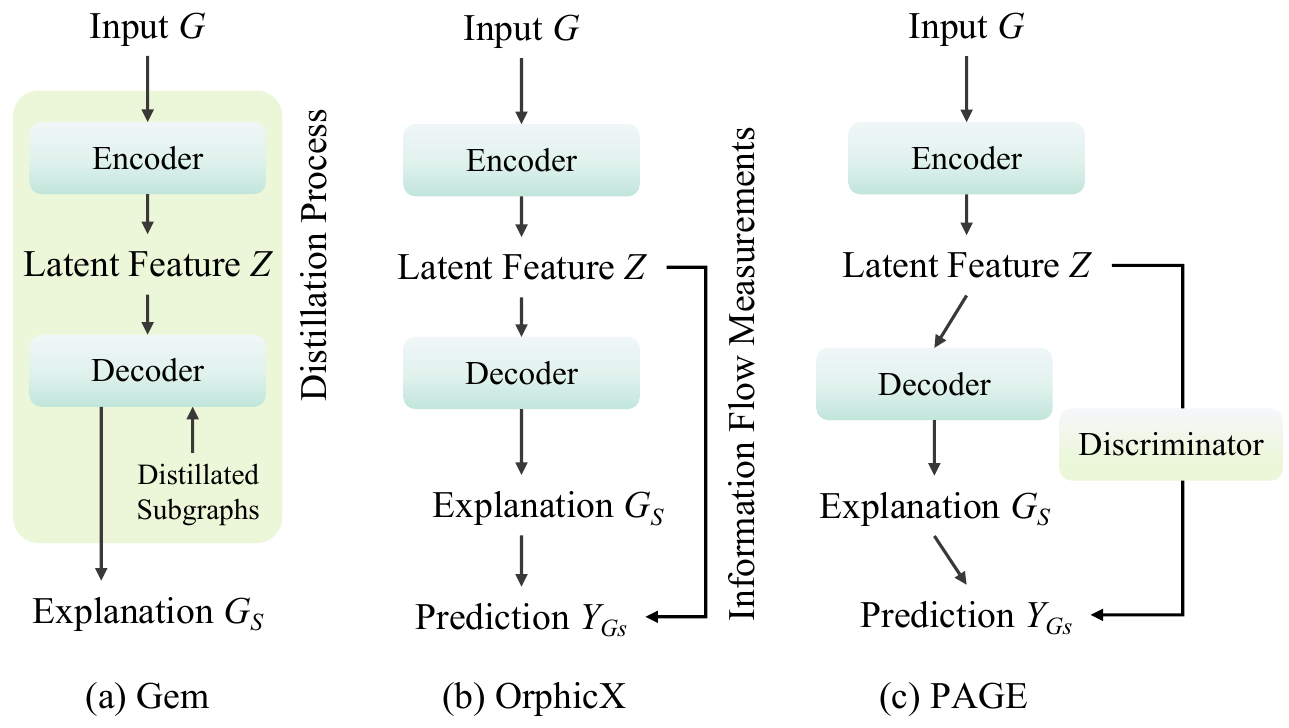}
    \caption{Methods using autoencoder as explainer: (a) Gem\cite{linGenerativeCausalExplanations2021}, (b) OrphicX\cite{linOrphicXCausalityInspiredLatent2022}, (c) PAGE(Ours)}
    \label{figure1}
    \vspace{0.8cm}
    \end{figure}

\section{Related Work}
\vspace{-0.02cm}
Post-hoc GNN Explanation aims to produce an explanation for a GNN prediction on a given graph, usually as a substructure of the graph. Various explaining approaches focus on different aspects of the model and also provide different views. Many surveys categorized and summarized existing GNN explaining methods~\citep{yuanExplainabilityGraphNeural2021, ratheeBAGELBenchmarkAssessing2022, liExplainabilityGraphNeural2022, baldassarreExplainabilityTechniquesGraph2019}. 
Some methods provide explanations for each individual input instance separately, such as GNNExplainer~\citep{yingGNNExplainerGeneratingExplanations2019a}, while some train a parameterized explainer to provide explanations collectively for multiple instances, like PGExplainer~\citep{luoParameterizedExplainerGraph2020a}. Depending on whether the explainer generates explanations separately or collectively, we categorize existing explanation methods into non-parameterized and parameterized approaches.

\textbf{Non-parameterized:} Non-parameterized methods like GNNExplainer generate explanations for individual instances and predictions through perturbation or optimization processes. Some gradient-based methods compute gradients of the target prediction with respect to input features through back-propagation. Moreover, feature-based methods map hidden features to the input space via interpolation to calculate importance scores. Examples of such methods include SA~\citep{baldassarreExplainabilityTechniquesGraph2019}, GuidedBP~\citep{baldassarreExplainabilityTechniquesGraph2019}, and Grad-CAM~\citep{popeExplainabilityMethodsGraph2019}. Additionally, there are some methods based on cooperative game theory, which assign importance scores to input nodes and edges by introducing concepts from cooperative games like Shapley Value. Examples include SubgraphX~\citep{yuanExplainabilityGraphNeural2021} and GStarX~\citep{zhangGStarXExplainingGraph2022a}. Moreover, some approaches integrate reinforcement learning by treating the addition or removal of nodes and edges as strategies, while considering the probabilities of the target GNN as rewards, and then utilize a policy network to generate explanations, aiming to maximize the attained rewards,like RC-Explainer~\citep{wangReinforcedCausalExplainer2023} and RG-Explainer~\citep{shanReinforcementLearningEnhanced2021}.

\textbf{Parameterized:} Parameterized methods like PGExplainer train an explainer on the entire training set to provide explanations for multiple instances collectively. 
These methods utilize a parameterized explaining network, typically trained to maximize the mutual information or causal attribution between the explanatory subgraph/structure and the prediction. Subsequently, the network is employed to collectively generate explanations, like Gem~\citep{linGenerativeCausalExplanations2021}, OrphicX~\citep{linOrphicXCausalityInspiredLatent2022} and our method PAGE. Some methods give explanations by developing interpretable surrogate models, which are subsequently employed to approximate the behavior of the target model to be explained. The explanations derived from these surrogate models are then used to interpret the behavior of the target GNN, such as GraphLime~\citep{huangGraphLIMELocalInterpretable2023}, PGM-Explainer~\citep{vuPGMExplainerProbabilisticGraphical2020a}, and GraphSVX~\citep{duvalGraphSVXShapleyValue2021}.

Beyond these two types of methods, there are also model-level explanation approaches that seek to explore the patterns of the input that can induce specific behaviors in the GNN. These approaches offer more general and high-level insights into the GNN model. Examples include XGNN~\citep{yuanXGNNModelLevelExplanations2020}, GLGExplainer~\citep{azzolinGlobalExplainabilityGNNs2023} and GNNInterpreter~\citep{wangGNNInterpreterProbabilisticGenerative2023}. 
Additionally, some methods specifically focus on offering counterfactual explanations, such as CF-GNNExplainer, CF2, and RCExplainer, and some self-explanatory models aim to develop the capability to generate predictions alongside corresponding explanations, like ProtGNN~\citep{zhangProtGNNSelfExplainingGraph2022}. 

\section{Methodology}

\begin{figure*}[t]
    \centering
    \includegraphics[width=2.0\columnwidth]{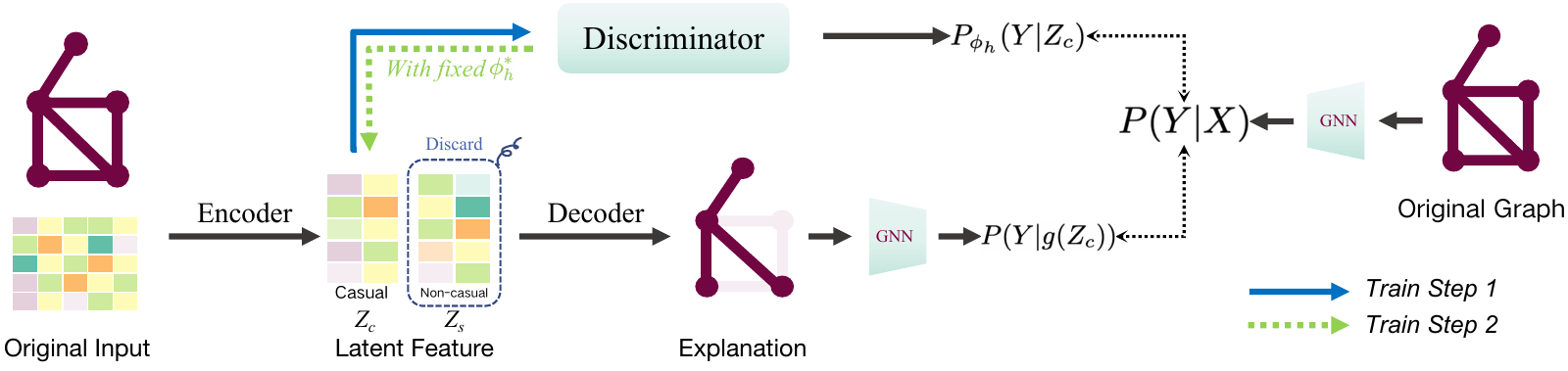}
    \caption{Basic framework of PAGE. The causal part of the latent features related to GNN prediction is used to generate explanations, while the non-causal part is discarded. An additional discriminator is employed to maximize the mutual information between causal features and prediction results. The training process consists of two steps: the first step (indicated by the blue arrows) involves training the autoencoder and the discriminator, and the parameters of the discriminator will be fixed to constrain the encoder in generating better causal features in the second step.}
    \label{figure2}
    \vspace{0.8cm}
    \end{figure*}

\subsection{Preliminary}

\subsubsection{Graph neural network}
The work of graph neural network mainly depends on two mechanisms: message passing and representation aggregation~\citep{JMLR:v12:shervashidze11a}~\citep{kipfSemiSupervisedClassificationGraph2017}. A graph $G$ can be described as $G = ( V , E )$ with node attribute $X$ and adjacency matrix $A$. Taking the Graph Convolutional Network(GCN)~\citep{kipfSemiSupervisedClassificationGraph2017} as an example, the computation process is as follows: 
\begin{equation}
h_{v}^{k+1}=\sigma\left(\sum_{u \in \mathcal{N}(v)}\left(W^{k} h_{u}^{k} \tilde{A}_{u v}\right)\right)
\end{equation}
where $h_u^k$ is the representation of node $u$ at the $k$ th layer in GCN and $\tilde{A}=\hat{D}^{-\frac{1}{2}} \hat{A} \hat{D}^{-\frac{1}{2}}$ is the normalized adjacency matrix. $\hat{A}=A+I$ is the adjacency matrix of the graph $G$ with self loops added and $\hat{D}$ is a diagonal matrix with $\hat{D}_{i i}=\sum_j \hat{A}_{i j}$. $\mathcal{N}(v)$ denote the neighbors of node $v$ and $W^k$ is the weight matrix to be trained of the $k$-th layer. Eq. (1) represents that node $v$ will collect and aggregate the information of all neighboring nodes, which will be used as the representation of node $v$ in the next layer through the activation function $\sigma(*)$. Nodes’ representation will be used to perform node-level tasks or complete graph-level tasks by using readout functions.

\subsubsection{Graph Auto-encoder}

Graph auto-encoder~\citep{kipfVariationalGraphAutoEncoders2016} is a self-supervised learning method. The model consists of a inference network and a generative network, and we call then the encoder and decoder respectively. It mainly retains the following two procedures: It maps the input sample to the hidden space through the GCN encoder to obtain the low-dimensional representation of the sample features, and then latent features are used to reconstruct the original graph structure through a inner product decoder. This process can be formulated as follows:
\begin{equation}
\hat{\mathnormal{A}}=\sigma\left(\mathnormal{Z} \mathnormal{Z}^{\top}\right), \textit{ with } \quad \mathnormal{Z}={GCN}(\mathnormal{X}, \mathnormal{A})
\end{equation}
For graph auto-encoder(GAE), the adjacency matrix reconstructed from $Z$ should be as comparable as possible to the original, so GAE uses cross-entropy as the loss function:
\begin{equation}
\mathcal{L}_{\mathrm{GAE}}=E_{q(Z \mid X, A)}[\log p(A\mid Z)]
\end{equation}
$q$ and $p$ are the encoding and decoding functions, respectively. For variational graph auto-encoder(VGAE), $Z$ is obtained by sampling from a Gaussian distribution instead of by a definite function. In VGAE (Variational Graph Autoencoder), the reparameterization technique is employed to facilitate the backpropagation of gradients. For more detailed information, you can refer to the original literature~\citep{kipfVariationalGraphAutoEncoders2016}. It uses two GCNs that share the same weight in the first layer to generates the means and variances. An additional Kullback-Leibler divergence more than GAE is employed in the loss function:
\begin{equation}
\mathcal{L}_{\mathrm{VGAE}}=E_{q(Z \mid X, A)}[\log p(A \mid Z)]-\mathrm{KL}[q(Z \mid X, A) \| p(Z)]
\end{equation}

\subsubsection{Auto-encoder as an explainer}
One application of autoencoder is to learn low-dimensional representations of the input data utilizing the nonlinear expressive power of neural networks. These low-dimensional representations(denoted as $Z$) in the latent space should be more distinguishable. Based on this principle, we aim to separate the causal feature component (denoted as $Z_c$) related to the model's predictions from the entire features in the latent space and discard the non-causal part(denoted as $Z_s$). This approach avoids separating the causal and non-causal parts in the input spaces as in previous methods. By imposing constraints on the latent space feature along the feature dimensions and setting an appropriate learning objective, we intend to isolate the causal sub-feature $Z_c$ from $Z$ in terms of the feature dimension. Then the inner product decoder can map the sub-matrix $Z_c$ into an adjacency mask, serving as a substructure for explanation.

\subsection{Framework}
As shown in Figure \ref{figure1} , our model mainly consists of three components: the encoder (denoted as $f(\cdot):\mathcal{G} \mapsto \mathcal{Z}$), decoder (denoted as $g(\cdot):\mathcal{Z} \mapsto \mathcal{G}$), and discriminator (denoted as $h(\cdot):\mathcal{Z} \mapsto \mathcal{Y}$). We mark the GNN to be explained as $\mathcal{F(\cdot)}:\mathcal{G} \mapsto \mathcal{Y}$, which gives a predicted label $Y\in \mathcal{Y} $ for each input graph $G\in \mathcal{G} $. The model learns the low-dimensional representation of the input graph through a two-layer GCN encoder. In order to maintain the structural consistency of the reconstructed graph with the original input, latent features $Z=f(A,X) \in\mathcal{Z}$ is employed to calculate the auto-encoder reconstruction loss. Meanwhile, we partition $Z$ into causal and non-causal parts based on the feature dimensions(denoted as $Z=cat(Z_c,Z_s)$). $Z_c$ will be concatenated with a zero matrix to restore the original dimensions of $Z$, and then utilized to generate an adjacency mask, serving as the explanation subgraph(denoted as $G_S, G_S=(A\odot g(Z_c),X)$). Then we have the following assumptions:
\begin{proposition}
    For any $Z_p$ that $Z_p \subseteq Z$, $I(Y;Z_p)\le I(Y;Z)$
\end{proposition}

\begin{proof}
    $I(Y;Z_p)\le I(Y;Z)$ is equivalent to $H(Y|Z_p)\ge H(Y|Z)$ where $H$ denotes entropy, since $H(Y)$ is a constant. $Z_p \subseteq Z$ means $Z_p$ is a subset of $Z$, i.e., $Z = Z_p\cup Z_+$. Therefore, $H(Y|Z_p)\ge H(Y|Z_p,Z_+) = H(Y|Z), i.e., I(Y;Z_p)\le I(Y;Z)$
\end{proof}

\begin{proposition}
    For causal features in the latent space, they should be strongly correlated with the model predictions(i.e., mutual information $I(Y;Z_c)=I(Y;Z)$ ). Then the training criterion of the encoder should be:
    \begin{equation}
        \theta _{E}^{*} =\underset{\theta_E}{\mathrm{argmax} } I(Y;Z_c)\text{, since~} I(Y;Z_c)\le I(Y;Z)
    \end{equation}
    \label{proposition2}
    \end{proposition}

\begin{proposition}
    For non-causal features in the latent space, they should be independent of the model predictions($Z_s\perp Y$, i.e., mutual information $I(Y;Z_s)=0 $).
    \end{proposition}

The above hypotheses can only hold true when the parameter of encoder~$\theta _{E}$ are optimal. Similar to previous work, we refer to the features in the latent space that are correlated with the predictions made by the GNN being explained as causal features (denoted as $Z_c$), and the parts that are not correlated as non-causal (denoted as $Z_s$). In other words, these features are the causal attribution behind the predictions made by the GNN model on the original input graph $G$.

However, the mutual information $I(Y;Z_s)$ and $I(Y;Z_c)$ cannot be directly calculated, since $ I(Y;Z_c)=\mathbb{E}\left [ \log_{}{\frac{P(Y|Z_c)}{P(Y)} }  \right ] $ and $P(Y|Z_c)$ is unknown (We can only obtain the distribution $P(Y|G_s))$ given by the target GNN). Following Proposition 1 in GAN~\citep{goodfellowGenerativeAdversarialNets2014}, we introduce an optimal discriminator $h: \mathcal{Z} \mapsto \mathcal{Y}$ with parameter $\phi_h$ to approximate $P(Y|Z_c)$ as $P_{\phi_{h}}(Y|Z_c)$. We then have the following proposition:


\begin{definition}
    For $\theta$ fixed, the optimal discriminator $\phi_{h}^*$ is:
    \begin{equation}
        \phi_{h}^* = \underset{\phi_h}{argmax} \mathbb{E}\left [ \log_{}{P(Y|Z_c)}  \right ]
    \end{equation}
    \label{definition1}
\end{definition}

\begin{proposition}
    Denoting Kullback-Leibler divergence as $KL\left [ \cdot\left  | \right | \cdot \right ] $, for $\theta$ fixed, the optimal discriminator $\phi_{h}$ is $\phi_{h}^*$, s.t.:
    \begin{equation}
        KL\left [P(Y|g(Z_c)) \left |  \right | P_{\phi_{h}}(Y|Z_c) \right ]=0
    \end{equation}
    \label{proposition3}
\end{proposition}

The proof of Proposition \ref{proposition3} is as follows:

\begin{proof}
    According to the definition of mutual information, we have:
    \begin{equation}
        \begin{aligned}
            I(Y;Z_c)=\mathbb{E}\left [ \log_{}{P_{\phi_{h}}(Y|Z_c)}  \right ]  +H(Y)+\\
            KL\left [P(Y|Z_c) \left |  \right | P_{\phi_{h}}(Y|Z_c) \right ]
        \end{aligned}
    \end{equation}
    For $\theta$ fixed, $I(Y;Z_c)$ and $H(Y)$ are constant, so we have:
    \begin{equation}
        \phi_{h}^*=\underset{\phi_h} {argmin}KL\left [P(Y|Z_c) \left |  \right | P_{\phi_{h}}(Y|Z_c) \right ]
    \end{equation}
\end{proof}

Thus, by disregarding the constant $H(Y)$, the training criterion becomes:
\begin{equation}
    \begin{split}
        \theta _{E}^{*}=\underset{\theta_E }{argmax}\left \{ \underset{\phi _{h}^{} }{max} \mathbb{E}\left [  \log_{}{P_{\phi_{h}}(Y|Z_c)} \right ]  \right \} 
    \end{split}
\end{equation}

It is difficult to directly find the optimal parameters that satisfy the above expression, so we propose a two-step training strategy:
In the first step, the discriminator is trained to make predictions consistent with the target GNN based on the causal sub-features $Z_c$ in the latent space. The discriminator is a two-layer Multi-Layer Perceptron (MLP). More details of our model can be found in the Appendix~\citep{qiu_2024_13375562}. Since the latent features are dimensionality-reduced representations, the discriminator can achieve high accuracy.
In the second step, we fix the parameters of the discriminator to enforce the causal features to be distributed as much as possible in the dimensions we require. Then, the features are used by the decoder to generate explanations based on these causal features.

\subsection{Learning Objectives}

Post-hoc explanation refers to the process of providing an explanation or justification for a well-trained model. The outputs of GNN $\mathcal{F}$ are regarded as the training label (formalized as $Y = \mathcal{F}(A, X)$). The learning objective is designed to maximizing the mutual information between the model predictions and the underlying structure $G_S$:

\begin{equation}
\max _{G_s} \operatorname{I}\left(Y; G_s\right)=\max _{G_s}\left \{H\left(Y \right)-H\left(Y  \mid G=G_s\right) \right\}
\end{equation}

Our training process consists of two stages. In the first stage, the discriminator and auto-encoder are trained together. The motivation is to ensure that the auto-encoder can completely rejuvenate the original graph structure and train the discriminator to learn the causal attribution from causal features to the predicted outcome. The learning objective can be formulized as:

\begin{equation}
    \begin{split}
        \mathcal{L}=\mathcal{L_\mathrm{AE}}&+\lambda_1 * KL\left [P_{\phi_{h}}(Y|Z_c) \left |  \right | P(Y|g(Z_c)) \right ]\\
    &+\lambda_2 * KL\left [P(Y|g(Z_c)) \left |  \right | P(Y|X) \right ]
    \end{split}
\end{equation}

$\mathcal{L_\mathrm{AE}}$ is the loss of the autocoder (GAE or VGAE). 

In the second stage, the parameters of the trained discriminator are fixed. The learning objective of the second stage is:

\begin{equation}
\begin{split}
\mathcal{L}=\mathcal{L_\mathrm{AE}}+\mathcal{L_\mathrm{size}}+
\lambda_3 * KL\left [P(Y|g(Z_c)) \left |  \right | P_{\phi_{h_{fixed}}}(Y|Z_c) \right ]
\end{split}
\end{equation}

\begin{equation}
\mathcal{L_\mathrm{size}} =| \frac{\left\|\mathnormal{A} * g(Z_c)\right\|_1}{\|\mathnormal{A}\|_1} - \gamma|
\end{equation}

$\mathcal{L_\mathrm{size}}$ is the size loss to ensure that the mask generated by the decoder are within a reasonable range. In our experiment, the value of $\gamma$ is set to 0.5 in our experiments. $\mathnormal{A}$ is the adjacency matrix and $\gamma$ is a hyper parameter. Causal features $Z_c$ are transformed into masks of adjacency matrix format by the inner product decoder, which should be multiplied with the original adjacency matrix to serve as explanations.

\section{Evaluation}

\subsection{Datasets}
To verify the effectiveness of our proposed model, we conducted many experiments on various datasets. Like other GNN interpretation methods, we employed widely used synthetic and real-world datasets. For the node classification task, we use \textbf{BA-Shapes} and \textbf{Tree-Cycles} presented in~\citep{yingGNNExplainerGeneratingExplanations2019a} with gound-truth explanations, and for the graph classification task, we use real-world datasets \textbf{MUTAG}~\citep{debnathStructureactivityRelationshipMutagenic1991} and \textbf{NCI1}~\citep{waleComparisonDescriptorSpaces2008}. More details can be found in Appendix~\citep{qiu_2024_13375562}.

\subsection{Metrics}

\subsubsection{Fidelity} 
Fidelity is a commonly used metric to evaluate the faithfulness of the explanations to the model, which is defined as the difference of predicted probability/accuracy between the original predictions and the new predictions of masked input features given by the explainer~\citep{ribeiroWhyShouldTrust2016a}~\citep{yuanExplainabilityGraphNeural2023}. Intuitively, the local important input features identified by the interpreter are discriminative to the GNN model. In that case, the model's prediction should change significantly when these local features are removed. We measure this changement with the fidelity score. Analogously, keeping only discriminative features should lead to similar predictions as the original, even if we remove the other unimportant features. We measure that variation by the in-fidelity metric. In our experiment, we used fidelity based on the predicted probability to verify how much the model can fit the behavior of the original GNN, and it is computed as:
\begin{equation}
\text { Fidelity }^{\text {prob }}=\frac{1}{N} \sum_{i=1}^{N}\left(\mathcal{F}\left(\mathcal{G}_{i}\right)_{y_{i}}-\mathcal{F}\left(\mathcal{G}_{i}^{1-m_{i}}\right)_{y_{i}}\right)
\end{equation}
\begin{equation}
\text { Infidelity }^{\text {prob }}=\frac{1}{N} \sum_{i=1}^{N}\left(\mathcal{F}\left(\mathcal{G}_{i}\right)_{y_{i}}-\mathcal{F}\left(\mathcal{G}_{i}^{m_{i}}\right)_{y_{i}}\right)
\end{equation}
Here $\mathcal{G}_{i}$ is the original graph and $\mathcal{F}$ is the GNN model to be explained. $\mathcal{G}_{i}^{1-m_{i}}$ represents the new graph obtained by keeping features of $\mathcal{G}_{i}$ based on the complementary mask ${1-m_{i}} $. $ \mathcal{G}_{i}^{m_{i}}$ is the new graph by keeping important features of $\mathcal{G}_{i}$ based on hard mask $m_{i}$ of the explanation.

\subsubsection{Accuracy} 
For synthetic datasets with ground truths, we can leverage the underlying rules of building these datasets to identify important edges or nodes, such as motifs of the graph. Using these important features as references, we can compare the explanations generated by the explainer with the ground truth. Accuracy, ROC, and F1 scores are commonly employed metrics for such evaluations.

The model accuracy measures the predictive accuracy of the generated explanations in relation to the original inputs. Specifically, we input both the original graph and the explanation into the target GNN model and compare the resulting predictions. A higher model accuracy indicates a closer alignment between the explanatory subgraph and the predicted outcomes of the original inputs. This demonstrates that the explainer is more proficient at identifying the most relevant subgraph for the pre-trained GNN.

\subsubsection{Sparsity} 
Sparsity measures the conciseness of explanations, as different explanation methods yield various forms of explanations. Some methods may select a subset of edges or nodes as explanations, while others may generate global importance scores for edges or nodes. Sparsity calculation refers to the proportion of explanations (e.g., edges or nodes) compared to the original input graph. The accuracy and fidelity of explanations are often closely related to sparsity. When sparsity is low, meaning the explanations are more comprehensive and closer to the original input, the accuracy and faithfulness of those explanations tend to be higher. To ensure a fair comparison, it is necessary to set an appropriate threshold that guarantees the sparsity of explanations remains within the same order of magnitude.

\subsection{Baselines}

We consider several baseline models, including perturbation-based method (GNNExplainer), parameterized model-based method (PGExplainer), and generative models based on autoencoder (Gem and OrphicX). 
PGExplainer, OrphicX, and Gem are all methods that train an interpreter to explain a target GNN model. On the other hand, GNNExplainer requires multiple perturbation processes on each input to generate explanations. Gem and OrphicX are the closest baselines to our proposed method in this paper. We set the hyperparameters of these baseline models according to the reported settings in their respective papers.

\subsection{Quantitive Analysis}

\begin{table}[htb]
\setlength{\tabcolsep}{1mm}{
\begin{tabular}{@{}l|lllll|lllll@{}}
\toprule[1.5pt]
K           & \multicolumn{5}{c|}{BA-SHAPES}                                                                                                                                             & \multicolumn{5}{c}{TREE-CYCLES}                                                                                                                                          \\ 
\# of edges & \multicolumn{1}{c|}{5}             & \multicolumn{1}{c|}{6}             & \multicolumn{1}{c|}{7}             & \multicolumn{1}{c|}{8}             & \multicolumn{1}{c|}{9} & \multicolumn{1}{c|}{6}             & \multicolumn{1}{c|}{7}             & \multicolumn{1}{c|}{8}            & \multicolumn{1}{c|}{9}            & \multicolumn{1}{c}{10} \\ \midrule
GNNExp.     & \multicolumn{1}{l|}{67.6}          & \multicolumn{1}{l|}{82.4}          & \multicolumn{1}{l|}{82.4}          & \multicolumn{1}{l|}{88.2}          & 85.3                   & \multicolumn{1}{l|}{64.3}          & \multicolumn{1}{l|}{66.5}          & \multicolumn{1}{l|}{74.3}         & \multicolumn{1}{l|}{88.6}         & 97.1                   \\
PGExp.      & \multicolumn{1}{l|}{53.4}          & \multicolumn{1}{l|}{59.5}          & \multicolumn{1}{l|}{60.8}          & \multicolumn{1}{l|}{65.5}          & 68.5                   & \multicolumn{1}{l|}{76.2}          & \multicolumn{1}{l|}{81.5}          & \multicolumn{1}{l|}{91.3}         & \multicolumn{1}{l|}{95.4}         & 97.1                   \\
Gem         & \multicolumn{1}{l|}{64.7}          & \multicolumn{1}{l|}{76.4}          & \multicolumn{1}{l|}{89.5}          & \multicolumn{1}{l|}{\textbf{91.1} }          & \textbf{91.1}                   & \multicolumn{1}{l|}{74.2}          & \multicolumn{1}{l|}{85.7}          & \multicolumn{1}{l|}{\textbf{100}} & \multicolumn{1}{l|}{\textbf{100}} & \textbf{100}           \\
OrphicX     & \multicolumn{1}{l|}{61.7}          & \multicolumn{1}{l|}{61.7}          & \multicolumn{1}{l|}{73.5}          & \multicolumn{1}{l|}{76.4}          & 76.4                   & \multicolumn{1}{l|}{74.2}          & \multicolumn{1}{l|}{82.8}          & \multicolumn{1}{l|}{97.1}         & \multicolumn{1}{l|}{\textbf{100}} & \textbf{100}           \\
OrphicX-0     & \multicolumn{1}{l|}{61.7}          & \multicolumn{1}{l|}{73.5}          & \multicolumn{1}{l|}{73.5}          & \multicolumn{1}{l|}{76.4}          & 76.4                   & \multicolumn{1}{l|}{74.2}          & \multicolumn{1}{l|}{85.7}          & \multicolumn{1}{l|}{97.1}         & \multicolumn{1}{l|}{97.1} & \textbf{100}           \\
PAGE-GAE    & \multicolumn{1}{l|}{76.4}          & \multicolumn{1}{l|}{76.4}          & \multicolumn{1}{l|}{76.4}          & \multicolumn{1}{l|}{89.5}          & 89.5                   & \multicolumn{1}{l|}{74.2}          & \multicolumn{1}{l|}{82.8}          & \multicolumn{1}{l|}{97.1}         & \multicolumn{1}{l|}{97.1} & \textbf{100}           \\
\textbf{PAGE}         & \multicolumn{1}{l|}{\textbf{91.1}} & \multicolumn{1}{l|}{\textbf{91.1}} & \multicolumn{1}{l|}{\textbf{91.1}} & \multicolumn{1}{l|}{\textbf{91.1}} & \textbf{91.1}          & \multicolumn{1}{l|}{\textbf{74.3}} & \multicolumn{1}{l|}{\textbf{88.6}} & \multicolumn{1}{l|}{\textbf{100}} & \multicolumn{1}{l|}{\textbf{100}} & \textbf{100}           \\ \bottomrule[1.5pt]
\end{tabular}}
\vspace{0.2cm}
\caption{Explanation Accuracy on Synthetic Datasets (\%)}
\label{table1}
\end{table}

\begin{table}[htb]
\setlength{\tabcolsep}{1mm}{
\begin{tabular}{@{}l|ccccc|ccccc@{}}
\toprule[1.5pt]
           & \multicolumn{5}{c|}{Mutagenicity}                                                                                                                                 & \multicolumn{5}{c}{NCI1}                                                                                                                                          \\
1-Sparsity & \multicolumn{1}{c|}{0.5}           & \multicolumn{1}{c|}{0.6}           & \multicolumn{1}{c|}{0.7}           & \multicolumn{1}{c|}{0.8}           & 0.9           & \multicolumn{1}{c|}{0.5}           & \multicolumn{1}{c|}{0.6}           & \multicolumn{1}{c|}{0.7}           & \multicolumn{1}{c|}{0.8}           & 0.9           \\ \midrule
GNNExp.    & \multicolumn{1}{c|}{65.0}          & \multicolumn{1}{c|}{66.6}          & \multicolumn{1}{c|}{66.4}          & \multicolumn{1}{c|}{71.0}          & 78.3          & \multicolumn{1}{c|}{64.2}          & \multicolumn{1}{c|}{65.7}          & \multicolumn{1}{c|}{68.6}          & \multicolumn{1}{c|}{75.2}          & 81.8          \\
PGExp.     & \multicolumn{1}{c|}{59.3}          & \multicolumn{1}{c|}{58.9}          & \multicolumn{1}{c|}{65.1}          & \multicolumn{1}{c|}{70.3}          & 74.4          & \multicolumn{1}{c|}{57.7}          & \multicolumn{1}{c|}{60.8}          & \multicolumn{1}{c|}{65.2}          & \multicolumn{1}{c|}{69.3}          & 71.0          \\
Gem        & \multicolumn{1}{c|}{\textbf{66.4}} & \multicolumn{1}{c|}{67.7}          & \multicolumn{1}{c|}{71.4}          & \multicolumn{1}{c|}{\textbf{76.5}} & 81.8          & \multicolumn{1}{c|}{\textbf{61.8}} & \multicolumn{1}{c|}{65.6}          & \multicolumn{1}{c|}{70.6}          & \multicolumn{1}{c|}{74.9}          & 83.9          \\
OrphicX    & \multicolumn{1}{c|}{61.9}          & \multicolumn{1}{c|}{66.6}          & \multicolumn{1}{c|}{\textbf{72.8}} & \multicolumn{1}{c|}{76.0}          & 80.1          & \multicolumn{1}{c|}{59.1}          & \multicolumn{1}{c|}{60.8}          & \multicolumn{1}{c|}{67.1}          & \multicolumn{1}{c|}{76.3}          & 79.3          \\
OrphicX-0     & \multicolumn{1}{l|}{65.6}          & \multicolumn{1}{l|}{66.1}          & \multicolumn{1}{l|}{70.8}          & \multicolumn{1}{l|}{73.5}          & 78.5                   & \multicolumn{1}{l|}{58.1}          & \multicolumn{1}{l|}{61.3}          & \multicolumn{1}{l|}{66.4}         & \multicolumn{1}{l|}{72.7} & 79.1          \\
PAGE-GAE     & \multicolumn{1}{l|}{61.7}          & \multicolumn{1}{l|}{66.9}          & \multicolumn{1}{l|}{70.5}          & \multicolumn{1}{l|}{76.4}          & 80.4                   & \multicolumn{1}{l|}{56.2}          & \multicolumn{1}{l|}{60.8}          & \multicolumn{1}{l|}{65.1}         & \multicolumn{1}{l|}{74.9} & 79.1           \\
\textbf{PAGE}       & \multicolumn{1}{c|}{63.5}          & \multicolumn{1}{c|}{\textbf{68.4}} & \multicolumn{1}{c|}{72.3}          & \multicolumn{1}{c|}{\textbf{76.5}} & \textbf{83.4} & \multicolumn{1}{c|}{58.1}          & \multicolumn{1}{c|}{\textbf{66.4}} & \multicolumn{1}{c|}{\textbf{70.8}} & \multicolumn{1}{c|}{\textbf{77.4}} & \textbf{85.4} \\ \bottomrule[1.5pt]
\end{tabular}}
\vspace{0.2cm}
\caption{Explanation Accuracy on Real-World Datasets (\%)}
\label{table2}
\end{table}

\begin{table}[htb]
\setlength{\tabcolsep}{2mm}{
\begin{tabular}{@{}l|c|c|c|c|c@{}}
\toprule[1.5pt]
Datasets    & GNNExp. & PGExp. & Gem  & OrphicX & \textbf{PAGE}  \\ \midrule
BA-SHAPES   & 90.8    & 76.3   & 65.4 & 94.2    & \textbf{94.2} \\ \midrule
TREE-CYCLES & 91.2    & 77.2   & 76.3 & 94.2    & \textbf{96.7}  \\ \bottomrule[1.5pt]
\end{tabular}}
\vspace{0.2cm}
\caption{Edge Accuracy of Explanation (\%)}
\label{table3}
\end{table}

\begin{figure}[htb]
\centering
\vspace{-1cm}
\setlength{\abovecaptionskip}{-1.5cm}
\setlength{\belowdisplayskip}{3pt} 
\includegraphics[width=1\columnwidth]{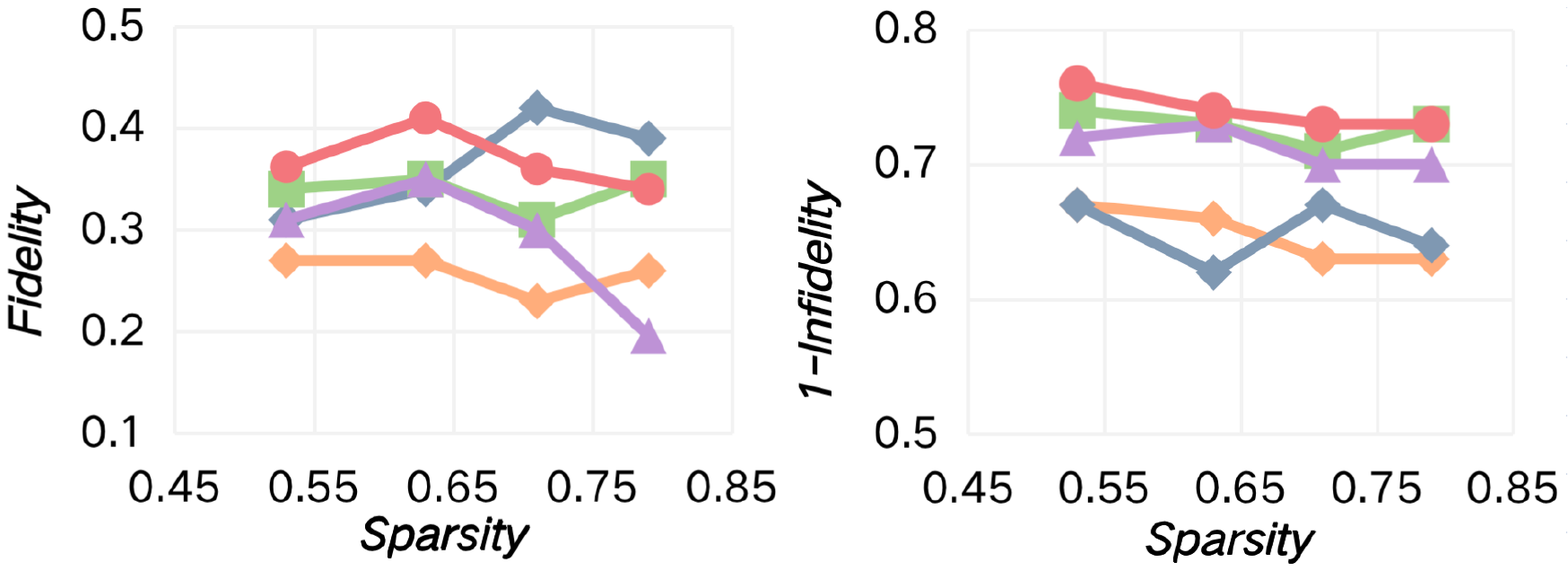} 
\vspace{0.5cm}
\caption{Fidelity and 1-Infidelity vs. Sparsity on BA-Shapes.}
\label{figure4}
\vspace{0.5cm}
\end{figure}

\begin{figure}[htb]
\centering
\vspace{-1cm}
\setlength{\abovecaptionskip}{-1.5cm}
\setlength{\belowdisplayskip}{3pt} 
\includegraphics[width=1\columnwidth]{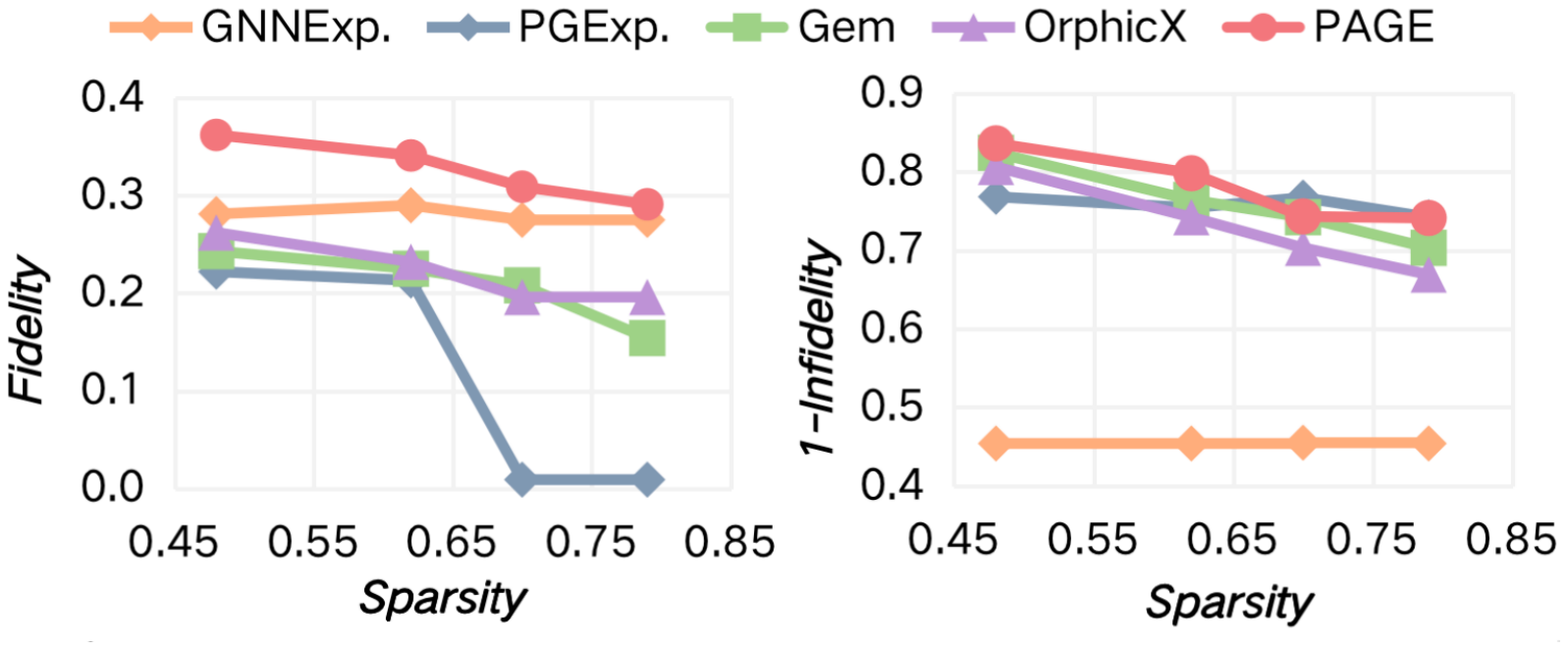} 
\vspace{0.5cm}
\caption{Fidelity and 1-Infidelity vs. Sparsity on Mutagenicity.}
\label{figure3}
\vspace{0.5cm}
\end{figure}

Similar to previous works~\citep{yuanExplainabilityGraphNeural2023}, our evaluation of the explainer's performance primarily focuses on two performance metrics: model accuracy and fidelity. Accuracy reflects how closely the GNN's explanations align with the original inputs. Table \ref{table1} and \ref{table2} respectively illustrate the performance of different explanation models on artificial synthetic datasets and real-world datasets under different sparsity constraints.

It should be noted that for the artificial synthetic dataset, we used a constraint of the top K edges for different models, instead of sparsity. This practice aligns with conventions in previous works. A unified K allows us to compare the concordance between explanations and ground truths. We didn't choose a value smaller than 5 for K because explanations formed by selecting fewer than 5 edges wouldn't comprehensively cover the ground-truth motifs, and such explanations would lack meaningful context. For instance, for the Tree-cycle dataset, choosing any k edges (k\textless5) from a cycle motif for explanation would not make any sense.

According to Table \ref{table1} and \ref{table2}, we observe that the accuracy of our explanations surpasses that of baseline models both on the synthetic datasets and real-world datasets. For the artificial synthetic dataset, our explanations achieve optimal accuracy under different constraints of K. For real-world datasets, our method outperforms other approaches under most sparsity constraints and maintains a higher average accuracy. We conducted experiments using both GAE and VGAE as interpreters. The experimental results indicated that the performance of GAE-based model was surpassed by the VGAE-based. Consequently, for the subsequent experiments, we consistently chose VGAE as the foundational model for our interpreter. In addition, OrphicX-0 in the table represents the results obtained by removing the calculation of information flow from the loss function of OrphicX. The table shows that the information flow within OrphicX does not effectively capture the causal effects between hidden features and model predictions. Even after it was removed, the performance of the explainer did not significantly decline.

Indeed, different explanations could potentially lead to similar classification results as the original samples. Therefore, to further compare the performances among   different explaining methods, we reported the evaluation of edge accuracy on the synthetic dataset. By transforming the inclusion of each edge in the explanation into a binary classification problem, we can assess the concordance between the generated edges and the ground truth motifs. The results are presented in Table \ref{table3}. A higher accuracy indicates that the explainer tends to assign higher importance scores to edges of the ground truth motifs/subgraphs. Experimental results demonstrate that explanations generated by our approach are superior to the others. 

Furthermore, we use the fidelity and infidelity metrics to compare the quality of explanations generated by different methods under four different sparsity levels. Fidelity and Infidelity metrics can provide an alternative perspective on the quality of explanations.We compare our model and baseline methods on the BA-Shapes and Mutagenicity datasets, as shown in Figure \ref{figure3} and \ref{figure4} (using 1-Infidelity as the y-axis for easy comparison), indicating that our method generally exhibits better fidelity in most cases.

\subsection{Qualitative Analysis}

\begin{figure*}[htb]
\centering
\includegraphics[width=1.8\columnwidth]{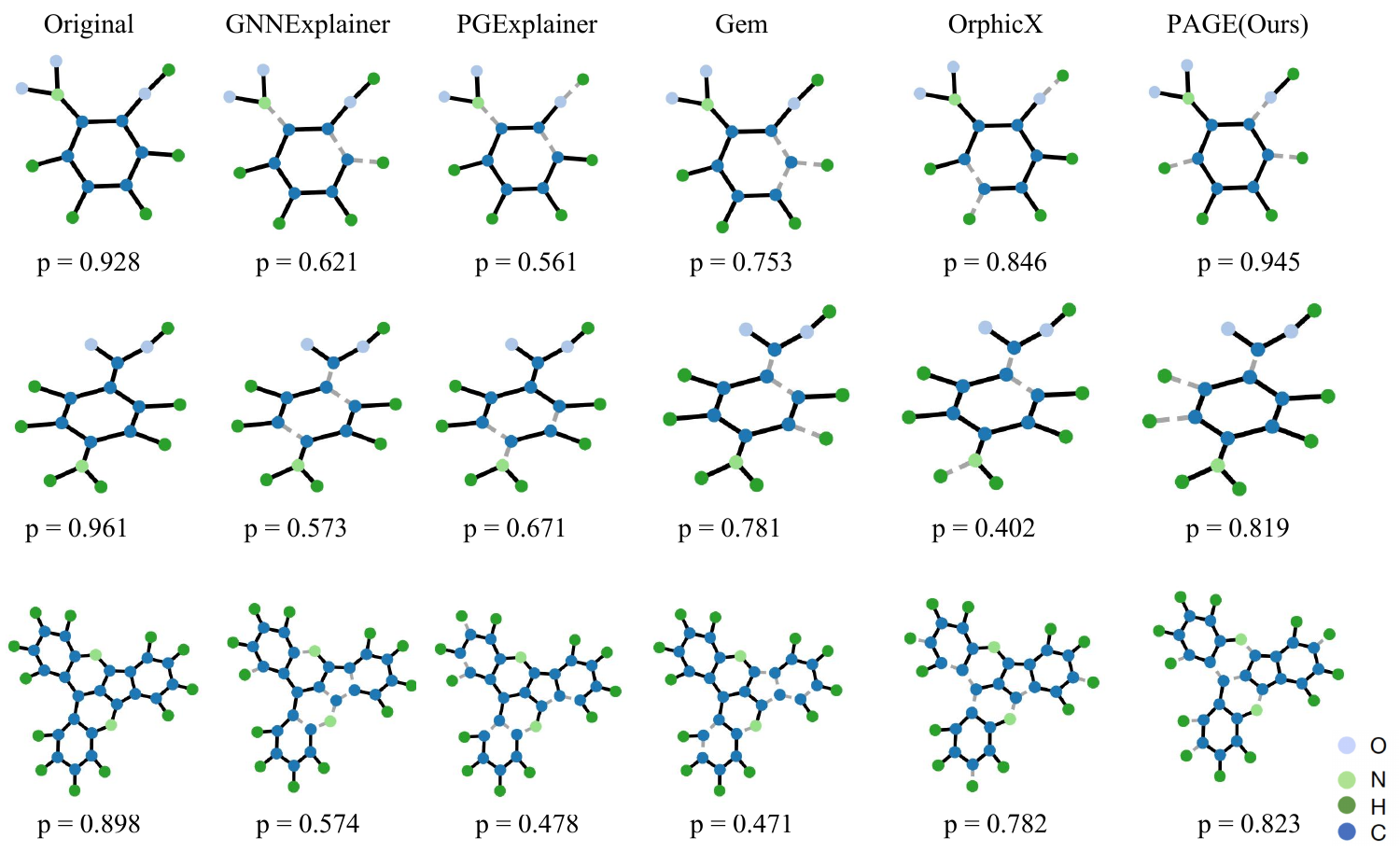} 
\caption{Visualized explanations on Mutagenicity. The explanation results of different methods are highlighted with black edges, where the gray edges are regarded as non-casual parts for the prediction. "P" under each graph/subgraph denotes the probability of being classified into the mutagenic class, which is obtained by feeding the associated graph/subgraph into the target GNN.}
\label{figure5}
\vspace{0.8cm}
\end{figure*}

To further demonstrate the interpretability of the explanations, Figure \ref{figure5} illustrates some visualized explanation instances generated by different methods on the Mutagenicity dataset. The figure presents three different mutagenic molecules. The first column represents the initial molecule, while the remaining five columns showcase explanations generated by different methods, all adhering to a common sparsity constraint. Black edges indicate the edges the interpreter selects, while gray edges signify those not selected. The value of p represents the probability that the initial molecule (or the explained) possesses mutagenic potential given by the target GNN. In the first two cases, our method is the only one capable of capturing the complete carbon ring and the connected -NO2 functional group attaching to it, which is always regarded as a feature of being mutagenic. In the second case, OrphicX failed to recognize the -NH2 functional group, and the explanation is predicted as non-mutagenic. In the third example, we report a mutagenic molecule with no explicit motifs. PAGE produced an explanation that most closely matched the original prediction. Such explanations could aid in uncovering novel chemical patterns. In summary, compared to baseline models, our approach generates explanations that best reflect the behavior of the target GNN.

\subsection{Ablation Study}
\begin{table}[htb]
    \setlength{\tabcolsep}{2.6mm}{
    \begin{tabular}{@{}l|lll|lll@{}}
        \toprule[1.5pt]
                        & \multicolumn{3}{c|}{BA-SHAPES(top k)}                        & \multicolumn{3}{c}{Mutagenicity(1-Sparsity)}                 \\
                        & \multicolumn{1}{l|}{6}    & \multicolumn{1}{l|}{7}    & 8    & \multicolumn{1}{l|}{0.6}  & \multicolumn{1}{l|}{0.7}  & 0.8  \\ \midrule
        None            & \multicolumn{1}{l|}{59.5} & \multicolumn{1}{l|}{59.5} & 59.5 & \multicolumn{1}{l|}{61.9} & \multicolumn{1}{l|}{62.4} & 65.2 \\ \midrule
        ERM             & \multicolumn{1}{l|}{76.4} & \multicolumn{1}{l|}{88.5} & 88.5 & \multicolumn{1}{l|}{64.6} & \multicolumn{1}{l|}{69.7} & 75.4 \\ \midrule
        No Pre-training & \multicolumn{1}{l|}{77.3} & \multicolumn{1}{l|}{89.5} & 89.5 & \multicolumn{1}{l|}{65.5} & \multicolumn{1}{l|}{69.7} & 76.7 \\ \midrule
        FULL            & \multicolumn{1}{l|}{\textbf{91.1}} & \multicolumn{1}{l|}{\textbf{91.1}} & \textbf{91.1} & \multicolumn{1}{l|}{\textbf{66.4}} & \multicolumn{1}{l|}{\textbf{70.8}} & \textbf{77.4} \\ \bottomrule[1.5pt]
        \end{tabular}}
        \vspace{0.2cm}
    \caption{Ablation Study Result. (\textit{1, None: only autoencoder loss and size loss.}\textit{2, ERM: the empirical term $KL\left [P(Y|g(Z_c)) \left |  \right | P(Y|X) \right ]$ added.}\textit{3, No pre-training: do not pre-train the discriminator.}\textit{4, FULL: fully trained.})}
    \label{table4}
    \end{table}

To further investigate the effectiveness of designed components in PAGE, we conducted ablation experiments on both the artificially synthesized dataset BAShapes and the real-world dataset Mutagenicity. Specifically, we validated the performance of the interpreter under four conditions. The result is shown in Table~\ref{table4}, showing that applying only partial component cannot obtain optimal performance. Additionally, we conducted experiments on hyperparameter sensitivity. Please refer to the supplementary materials for more details.
    
\subsection{Efficiency Study}

\begin{table}[htb]
\setlength{\tabcolsep}{2.8mm}{
\begin{tabular}{@{}c|c|c|c|c|c@{}}
\toprule[1.5pt]
time(ms)  & GNNExp. & PGExp.  & Gem     & OrphicX & \textbf{PAGE}    \\ \midrule
training     & -       & 340,148 & 24,041 & 368,718 & \textbf{78,233} \\ \midrule
inference & 533,681  & 354     & 79      & \textbf{77}      & 91     \\ \midrule
total     & 533,681  & 340,502 & 24,120 & 368,795 & \textbf{78,324} \\ \bottomrule[1.5pt]
\end{tabular}}
\vspace{0.2cm}
\caption{Training and inference time(ms) on Mutagenicity.}
\label{table5}
\end{table}

GNNExplainer requires multiple perturbations on a single sample. PGExplainer necessitates generating soft masks individually for each edge. GEM involves considering each edge to obtain a subgraph for the "guidance" of the training. OrphicX computes information flow through sampling at different scales. Compared to these methods, our approach eliminates any perturbation or sampling processes. Each inference and backpropagation involve only a single computation, resulting in a time complexity of O(1). As a result, our method remains significantly more efficient than these baseline models. Table \ref{table4} presents the training and inference times for the mentioned models on the Mutagenicity dataset. All models are configured with hyperparameters, learning rates, and epoch numbers as described in their original papers. It needs to be specified that, although experiments demonstrate that Gem has shorter training and inference times compared to ours, Gem requires an additional distillation process to generate guidances for the training of the autoencoder. This distillation incurs significant costs (more than 150,000 ms on the Mutagenicity dataset). Consequently, considering the overall expenses, our method's efficiency still surpasses all other baseline models.

\section{Conclusion}

In this article, we introduce PAGE, a parametric generative Graph Neural Network (GNN) explaining method designed to generate concise and reliable causal explanations for any graph neural network. PAGE optimize a generative autoencoder with a learning objective of maximizing mutual information between latent features and outputs. Compared to existing methods, PAGE offers several advantages: its computation and inference processes do not require any perturbation or sampling processes, ensuring high explanation accuracy while maintaining greater efficiency than previous approaches. Moreover, as a model-agnostic post-hoc explanation approach, it can offer causal explanations for different types of GNNs without relying on any prior assumptions or internal model details. We demonstrate the superiority of our approach over baseline models through experiments and data analysis. A limitation of PAGE is that we only explored autoencoders as interpreters. An avenue for potential improvement could involve using more powerful generative models as interpreters. We leave this for future investigation.

\begin{ack}
    This work is supported by National Natural Science Foundation of China under grants 62376103, 62302184, 62206102, Science and Technology Support Program of Hubei Province under grant 2022BAA046, and CCF-AFSG Research Fund.
    \end{ack}


\clearpage

\bibliography{mybibfile}

\clearpage

\begin{appendix}
\section{Experiments details in Section 3.}

\subsection{Datasets}
BA-Shapes and Tree-cycles are synthetic node classification datasets. BA-Shapes is composed of 80 house motifs and a base BarabasiAlbert (BA) graph containing 300 nodes. Node labels are divided into four categories, representing vertices, middle, and bottom parts of the houses, or not belonging to any motifs. Tree-cycles consist of a base eight-level binary balanced trees with 80 six-node cyclic motifs. Node labels are binary and denote whether the node belongs to the cycle motif. Mutagenicity and NCI1 are real-world graph classification datasets. Mutagenicity comprises 4337 unique chemical molecules, where nodes represent individual atoms and edges denote different chemical bonds. The labels indicate whether the chemical molecule is mutagenic. NCI1 includes 4110 chemical compounds with labels indicating whether the compound inhibits cancer cell growth. Table \ref{table5} provides additional details about these datasets.

\subsection{Experimental Hardware}
All experiments were conducted on a PC equipped with an NVIDIA RTX 4070 Ti GPU and an Intel Core i5-13600KF processor, supported by 32GB of RAM and 12GB of graphics memory.
\subsection{Details about the target GNN}
For the target GNN to be explained, we followed the same experimental setup as previous works. Specifically, for node classification, we employed three layers of Graph Convolutional Networks (GCNs) with output dimensions set to 20. We concatenated the outputs from these three layers and subsequently subjected them to a linear transformation to derive the node labels. In the case of graph classification, we utilized three layers of GCNs with dimensions of 20 and conducted global max-pooling to generate the graph representations. A subsequent linear transformation layer was employed to derive the graph labels. The target GNN achieved accuracies of 94.1, 97.1, 88.5, and 78.6 on the BA-Shapes, Tree-cycles, Mutagenicity, and NCI1 datasets, respectively.
\subsection{Details about the explainer}
For the encoder of explainer, we applied a three-layer GCN with output dimensions 32, 32, and 16. The decoder is equipped with a two-layer MLP and an inner product decoder. The discriminator was implemented using a two-layer MLP with a hidden layer dimension of 32. We trained the explainers using a learning rate of 0.003 for 300 epochs in the first stage and 50 epochs in the second stage. All of our experiments and models, including the target GNN, our interpreter, and the baseline models, were implemented using PyTorch and trained with Adam optimizer.
\begin{table}[tbp]
\setlength{\tabcolsep}{2mm}{
\begin{tabular}{@{}c|c|c|c|c@{}}
\toprule[1.5pt]
Datasets     & BA-Shapes & Tree-cycles & Mutagenicity & NCI1        \\ \midrule
\#Graphs     & 1         & 1           & 4337         & 4110        \\ \midrule
\#Nodes      & 700       & 871         & 30.32(Avg.)  & 32.30(Avg.) \\ \midrule
\#Edges      & 4110      & 1950        & 30.77(Avg.)  & 29.87(Avg.) \\ \midrule
\#Labels     & 4         & 2           & 2            & 2           \\ \midrule
\#Training   & 300       & 270         & 3468         & 3031        \\ \midrule
\#Validation & 50        & 45          & 434          & 411         \\ \midrule
\#Testing    & 50        & 45          & 434          & 410         \\ \bottomrule[1.5pt]
\end{tabular}}
\vspace{0.2cm}
\caption{Details about the datasets.}
\label{table6}
\vspace{0.5cm}
\end{table}

\subsection{Downstream Tasks}

\textbf{Node-Level Task} 
For node-level tasks, such as node classification, the interpretation of the target node is designed as the subgraph with the strongest correlation with the prediction result on the computation graph. For a graph with more than thousands of nodes, it is impractical to reconstruct the complete graph for interpreting a single node; Moreover, the auto-encoder can not complete the training based on only one single graph. Therefore, for node-level tasks, we will obtain the subgraph composed of $k$-hop neighbors of each node for training, leading to a compact study on the whole graph (The value of $k$ is set as the number of layers in the target GNN.). When interpreting a specific node, we will use the subgraph consisting of no more than $k$-hop reconstructed by the encoder for prediction, and these subgraphs are considered as the most relevant and influential part of the input graph to the prediction.

\textbf{Graph-Level Task}
For graph-level tasks, such as graph classification, we will let the auto-encoder learn to reconstruct the complete graph structure. When making predictions, a readout function (can be average pooling or maximal pooling) is employed to obtain the graph-level representation from the latent feature $Z$. Then the discriminator employs the graph representation to make predictions. Experimental results show that this austere readout function is efficient and effective under the constraints of our designed optimizing objective.

\subsection{Setting of sparsity hyperparameter $\gamma$}

Regarding the setting of sparsity, although we recommend setting the sparsity hyperparameter $\gamma$ to 0.5 to generate an importance score matrix that is relatively balanced, we still suggest adjusting the sparsity differently based on different interpretability requirements to achieve better interpretability (for example, when a concise explanation is needed, lowering the sparsity as much as possible).

\subsection{Hyperparameter sensitivity.}

\begin{figure*}[htb]
    \vspace{-18cm}
    \includegraphics[width=2\columnwidth]{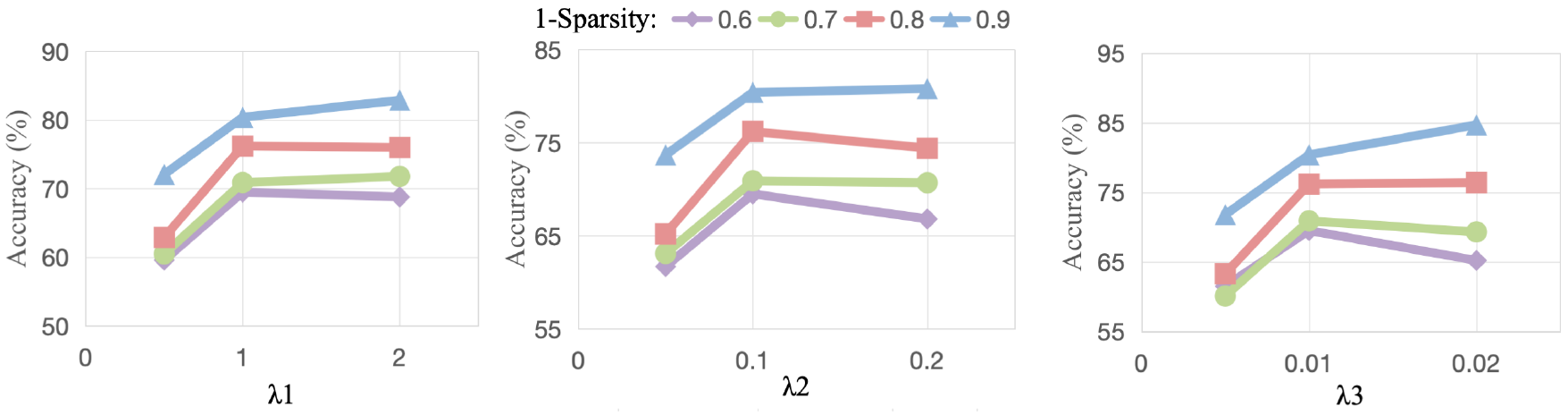} 
    \caption{Results about hyperparameter sensitivity.}
    \label{figure6}
    \vspace{0.8cm}
    \end{figure*}

To analyze the impact of different hyperparameter settings on the performance of PAGE, we conducted hyperparameter sensitivity experiments on Mutagenicity. We separately varied the values of the hyperparameters $\lambda_1$, $\lambda_2$, and $\lambda_3$ while keeping the default parameter settings ($\lambda_1=1.0$, $\lambda_2=0.1$, $\lambda_3=0.01$), and observed the explanation accuracy under different sparsity levels. The experimental results, as shown in Figure \ref{figure6}, indicate that all three hyperparameters have varying influence on the model's performances. This not only demonstrates the effectiveness of the discriminator module we designed but also underscores the necessity of selecting appropriate values for hyperparameters.

\end{appendix}

\end{document}